\documentclass{article}

\usepackage{hyperref}

\usepackage{amsmath}
\usepackage{amsthm}

\newtheorem{definition}{Definition}
\newtheorem{theorem}{Theorem}
\newtheorem{proposition}{Proposition}

\usepackage{tikz}
\usepackage{pgfplots}
\usetikzlibrary{positioning,chains,fit,shapes,calc}

\usepackage{arydshln}
\usepackage{url}
\usepackage{color}

\usepackage{algorithm2e}
\RestyleAlgo{boxruled}

\usepackage{setspace}

\usepackage{authblk}

\usepackage[square,numbers]{natbib}
\usepackage{geometry}
\title{The \textsc{bin\_counts} Constraint: \\ Filtering and Applications}
\author[1]{Roberto Rossi\thanks{Corresponding author. Address: 29 Buccleuch place, EH89JS, Edinburgh, UK. Email: roberto.rossi@ed.ac.uk}}
\author[2]{\"Ozg\"ur Akg\"un}
\author[3]{Steven Prestwich}
\author[4]{Armagan Tarim}

\affil[1]{Business School, University of Edinburgh, Edinburgh, UK}
\affil[2]{Department of Computer Science, University of Saint Andrews, UK}
\affil[3]{Insight Centre for Data Analytics, University College Cork, Ireland}
\affil[4]{Department of Management, Cankaya University, Turkey}
\date{}                                           

\begin{document}
\maketitle

\begin{abstract}
We introduce the \textsc{bin\_counts} constraint, which deals with the problem of counting the number of decision variables in a set which are assigned values that lie in given bins. We illustrate a decomposition and a filtering algorithm that achieves generalised arc consistency. We contrast the filtering power of these two approaches and we discuss a number of applications. We show that \textsc{bin\_counts} can be employed to develop a decomposition for the $\chi^2$ test constraint, a new statistical constraint that we introduce in this work. We also show how this new constraint can be employed in the context of the Balanced Academic Curriculum Problem and of the Balanced Nursing Workload Problem. For both these problems we carry out numerical studies involving our reformulations. Finally, we present a further application of the $\chi^2$ test constraint in the context of confidence interval analysis.\\
{\bf Keywords:} \textsc{bin\_counts}; generalised arc consistency; decomposition; statistical constraint; $\chi^2$ test constraint.
\end{abstract}

\newpage

\section{Introduction}
In Constraint Programming (CP) \cite{1207782} counting constraints and value constraints represent two important constraint classes which can be used to model a wide range of practical problems in areas such as scheduling and rostering.

Given a list of numbers, counting the number of elements in the list whose values lie in successive bins of given widths represents a problem that is often faced in spreadsheet modeling and mathematical computation. Tools that deal with this problem are available in most modern spreadsheets and mathematical computation programs. In Excel\textsuperscript{TM}, this modeling tool takes the form of command {\verb FREQUENCY };  in Mathematica\textsuperscript{TM}, of command {\verb BinCounts }. 

In this work we introduce the \textsc{bin\_counts} constraint, which models the aformentioned problem in a declarative fashion. To the best of our knowledge, no comparable constraint exists in state-of-the-art CP solvers.

We first introduce a decomposition for this constraint, as well as a polynomial filtering algorithm that achieves generalized arc consistency. We then present a number of applications for this new constraints. In our first application, we show how \textsc{bin\_counts} can be used to develop a decomposition for a new statistical constraint \cite{rossi2014}: the $\chi^2$ test constraint. In our second and third applications, we demonstrate the advantage of enforcing generalised arc consistency in the context of two standard CP benchmark problems: the balanced academic curriculum problem \cite{DBLP:journals/corr/cs-PL-0110007} and the balanced nursing workload problem \cite{10.2307/822876}. Finally, we present a further application of a variant of the $\chi^2$ test constraint in the context of confidence interval analysis.

This work is structured as follows. In Section \ref{sec:formal_background} we provide relevant notions on CP. In Section \ref{sec:bincount} we introduce the \textsc{bin\_counts} constraint. In Section \ref{sec:decomposition} we discuss a decomposition for this constraint and in Section \ref{sec:gac} we introduce a polynomial filtering algorithm that achieves generalised arc consistency. In Section \ref{sec:computational_study} we contrast performance of these two approaches. In Section \ref{sec:applications} we discuss applications of the constraint. In Section \ref{sec:related_works} we survey related works in the literature and in Section \ref{sec:conclusions} we draw conclusions.

\section{Formal background}\label{sec:formal_background}

A Constraint Satisfaction Problem (CSP) is a triple $\langle V,C,D \rangle$, where $V$
is a set of decision variables, $D$ is a function mapping each element of $V$ to a domain of
potential values, and $C$ is a set of constraints stating allowed combinations of values for subsets of variables in $V$ \cite{1207782}. A {\em solution} to a CSP is an assignment of variables to values in their respective domains such that all of the constraints are satisfied. 

There are different kinds of constraints used in CP: e.g. logic constraints, linear constraints, and {\em global constraints} \cite{reg03}. A global constraint captures a relation among a non-fixed number of variables. 
Constraints typically embed dedicated {\em filtering algorithms} able to remove provably infeasible
or suboptimal values from the domains of the decision variables that are constrained and, therefore, to enforce some degree of {\em consistency}, e.g. arc consistency (AC) \cite{citeulike:6598311}, bounds consistency (BC) \cite{Choi06} or generalised arc consistency (GAC) \cite{Freuder:1982:SCB:322290.322292}. A constraint is \emph{generalized arc-consistent} if and only if, when a variable is assigned any of the values in its domain, there exist compatible values in the domains of all the other variables in the constraint. A classical example of an arc consistency algorithm is the one presented in \cite{REGIN:1994:FAC:2891730.2891786} for the \textsc{all\_different} constraint, which constrains a given number of decision variables to take values that are all differents.
Filtering algorithms are repeatedly called until no more values are pruned; this process is called {\em constraint propagation}.

In addition to constraints and filtering algorithms, constraint solvers also feature a heuristic {\em search engine}, e.g. a backtracking algorithm guided by dedicated variable and value selection heuristics. During search, the constraint solver explores partial assignments and exploits filtering algorithms in order to proactively prune parts of the search space that cannot lead to a feasible or to an optimal solution.

\section{The {\normalfont\scshape bin\_counts} constraint}\label{sec:bincount}

Consider a list of $n$ integer values and a set of $m$ bins covering intervals $[b_j,b_{j+1})$, $j=1,\ldots,m$, our aim is to count, for each bin, the number of elements in the list whose values lie in it. 

{\bf Example.} Consider the following list of $n=10$ values $\{1, 1, 5, 3, 1, 2, 1, 1, 3, 1\}$ and a set of $m=3$ bins covering intervals $[1,3)$, $[3,4)$, $[4,6)$. By using command {\verb BinCounts } in Mathematica\textsuperscript{TM} we obtain the following counts for the 3 bins considered: $\{7,2,1\}$.

Without loss of generality and for ease of exposition, in what follows, we will keep referring to the case in which values and bin sizes are integer. However, it should be noted that the \textsc{bin\_counts} constraint can be extended to the case in which both values and bin sizes are real values; the generalized arc consistency propagator in Section \ref{sec:gac} seamlessly applies to this case.

Now, let $b_1,\ldots,b_{m+1}$ be scalar values, $x_i$ for $i=1,\ldots,n$ be decision variables with domain $\text{Dom}(x_i)$ and $c_j$ for $j=1,\ldots,m$ be decision variables with domain $\text{Dom}(c_j)$.

\begin{definition}
$\textsc{bin\_counts}_{b_1,\ldots,b_{m+1}}(x_1,\ldots,x_n;c_1,\ldots,c_m)$ holds iff $c_j$ is equal to the count of values assigned to $x_1,\ldots,x_n$ which lie within interval $[b_j,b_{j+1})$.
\end{definition} 

\section{A decomposition strategy}\label{sec:decomposition}

The $\textsc{global\_cardinality}_{v_1,\ldots,v_m}$($x_1,\ldots,x_n;c_1,\ldots,c_m$) constraint \cite{charme89} requires that, for each $j=1,\ldots,m$, decision variable $c_j$ is equal to the number of variables $x_1,\ldots,x_n$ that are assigned scalar $v_j$. 

We decompose \textsc{bin\_counts} by means of $b_{m+1}-b_1$ auxiliary variables, a \textsc{global\_cardinality} constraint, and a set of linear equalities as shown in Fig. \ref{model:bincounts_decomposition}. Essentially, we count occurrences $o_k$ of individual values $k\in\bigcup_{j=1}^m [b_j,b_{j+1})$ that appear in any of the bins (constraint 1). Then we sum all occurrences that belong to a given bin (constraint 2). Finally, we make sure that bin counts $c_j$ sum to $n$ (constraint 3). Note that this latter constraint can be formulated as an inequality ($\leq$) if we want to allow the $x_i$ to take values that fall outside the range of values covered by bins.

\begin{figure}[h!]
\begin{center}
        \framebox{
        \begin{tabular}{ll}
        \mbox{\bf Constraints:}\\
        \multicolumn{2}{l}{
        ~~~(1)~~$\textsc{global\_cardinality}_{b_1,\ldots,b_{m+1}}(x;o)$}\\
        ~~~~(2)~~$\sum^{b_{j+1}-1}_{k=b_j}o_k=c_j$						&$j=1,\ldots,m$\\
        ~~~~(3)~~$\sum_{j=1}^m c_j=n$\\
        \\     
        \mbox{\bf Parameters:} \\
        ~~~~$b_1,\ldots,b_{m+1}$ 			&bin boundaries\\
        ~~~~$n$							&number of value variables\\                        							
        ~~~~$m$							&number of bins\\                        
        \\
        \mbox{\bf Decision variables:} \\
        ~~~~$x_i$							&value variables\\        
        ~~~~$c_j$ 						&bin counts\\ 
        ~~~~$o_k$						&occurrences of value $k$\\ 
        \end{tabular}
        }
    \caption{\textsc{bin\_counts} decomposition}
    \label{model:bincounts_decomposition}
\end{center}    
\end{figure}

{\bf Example.} Consider the following numerical example with $n=3$ variable $x_i$, such that $\text{Dom}(x_1)=\{3,4\}$, $\text{Dom}(x_2)=\{1,2,4\}$, and $\text{Dom}(x_3)=\{2,3,4\}$; and $m=2$ bins ($b_1=1$, $b_2=3$, and $b_3=5$) where $\text{Dom}(c_1)=\{1,2,3\}$ and $\text{Dom}(c_2)=\{0,1\}$. We apply constraint propagation to the above decomposition by enforcing GAC on each constraint until a fixed point in reached. After filtering, domains are reduced as follows: $\text{Dom}(x_1)=\{3,4\}$, $\text{Dom}(x_2)=\{1,2,4\}$, $\text{Dom}(x_3)=\{2,3,4\}$, $\text{Dom}(c_1)=\{2,3\}$, and $\text{Dom}(c_2)=\{0,1\}$.

As we will show in the next section, these domains can be further reduced, therefore this decomposition does not achieve GAC. We shall next introduce GAC filtering for the \textsc{bin\_counts}.

\section{Enforcing generalized arc consistency}\label{sec:gac}

In this section we first illustrate theoretical properties and then present our filtering strategy that achieves GAC.

\subsection{Theoretical properties}

We reformulate \textsc{bin\_counts} as a bipartite network flow problem \citep{citeulike:14221207}. We generate a bipartite graph with one set of vertexes $v^x_1,\ldots,v^x_n$ and another set of vertexes $v^c_1,\ldots,v^c_m$. An arc $(v^x_i,v^c_j)$ exists between node $v^x_i$ and node $v^c_j$ if and only if there exists at least a value in the domain of $x_i$ which falls in bin $[b_j,b_{j+1})$. We label arc $(v^x_i,v^c_j)$ with the set of relevant values from $\text{Dom}(x_i)$ that fall within bin $[b_j,b_{j+1})$; note that these labels do not reflects arc capacities, which are all set to 1 for arcs $(v^x_i,v^c_j)$. We add a source node $s$ linked to nodes $v^x_1,\ldots,v^x_n$ and a terminal node $t$ linked to nodes $v^c_1,\ldots,v^c_m$; arc flows from $s$ to $x_i$ are set to 1; flow $c_j$ passing through arc $(v^c_j,t)$ must take a value in $\text{Dom}(c_j)$. We shall initially assume that all $\text{Dom}(c_j)$ are compact; this assumption will be relaxed at the end of this section.

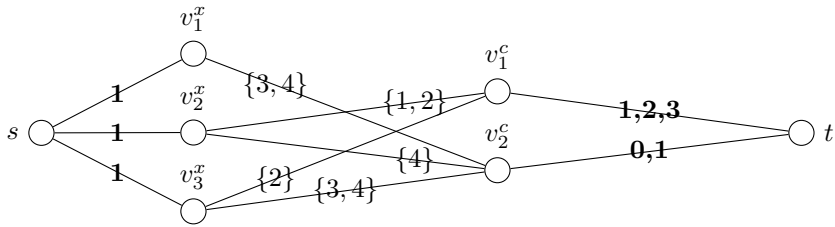
\begin{figure}
\center
\begin{tikzpicture}

\begin{scope}[xshift=-2cm,yshift=-1.05cm,start chain=going below, node distance=7mm]
\node[draw,circle] (s) [label=left: $s$] {};
\end{scope}

\begin{scope}[start chain=going below,node distance=7mm]
\foreach \i in {1,2,...,3}
  \node[on chain,draw,circle] (x\i) [label=above: $v^x_\i$] {};
\end{scope}

\begin{scope}[xshift=4cm,yshift=-0.5cm,start chain=going below,node distance=7mm]
\foreach \i in {1,2}
  \node[on chain,draw,circle] (c\i) [label=above: $v^c_\i$] {};
\end{scope}

\begin{scope}[xshift=8cm,yshift=-1.05cm,start chain=going below, node distance=7mm]
\node[draw,circle] (t) [label=right: $t$] {};
\end{scope}

\draw (s) -- (x1) node[midway] {\bf 1};
\draw (s) -- (x2) node[midway] {\bf 1};
\draw (s) -- (x3) node[midway] {\bf 1};

\draw (x1) -- (c2) node[near start] {$\{3,4\}$};
\draw (x2) -- (c2) node[near end] {$\{4\}$};
\draw (x3) -- (c2) node[midway] {$\{3,4\}$};

\draw (x2) -- (c1) node[near end] {$\{1,2\}$};
\draw (x3) -- (c1) node[near start] {$\{2\}$};

\draw (c1) -- (t) node[midway] {\bf 1,2,3};
\draw (c2) -- (t) node[midway] {\bf 0,1};
\end{tikzpicture}
\caption{Bipartite network flow graph for the numerical example; arc $(v^x_i,v^c_j)$ labels are shown in curly brackets, arc flows are shown in bold; we omit arc $(v^x_i,v^c_j)$ capacities, which are all set to 1}
\label{fig:bipartite}
\end{figure}

{\bf Example.} Consider once more the example introduced in Section \ref{sec:decomposition}.
The associated bipartite graph is shown in Fig. \ref{fig:bipartite}.

The feasible region of \textsc{bin\_counts} can be expressed as a system of linear equations. In the aforementioned graph theoretical construct let $f_{ij}$ denote the flow from $v^x_i$ to $v^c_j$. The linear inequalities that define our problem are
\begin{eqnarray}
&\sum_{j=1}^m f_{ij} 		= 1				&i=1,\ldots,n	\label{eq:1}\\
&\sum_{i=1}^n f_{ij} 	\leq \overline{c}_j		&j=1,\ldots,m	\label{eq:2}\\
&\sum_{i=1}^n f_{ij}	\geq \underline{c}_j		&j=1,\ldots,m	\label{eq:3}
\end{eqnarray}
where $\overline{c}_j$ and $\underline{c}_j$ represent the upper and the lower bounds of the domain of $c_j$, respectively. 

Note that the fact that a decision variable $x_i$ is assigned a value $v$, where $v$ falls within bin $[b_j,b_{j+1})$, can be expressed by adding a constraint $f_{ij}=1$. Similarly, if during search constraint propagation on $x_i$ filters all values listed in the label of arc $(v^x_i,v^c_j)$,\footnote{The removal of all values in the label of arc $(v^x_i,v^c_j)$ does not necessarily lead to a wipeout of $\text{Dom}(x_i)$ since some values of the domain may appear in the label of arcs $(i,k)$, where $k\neq j$.} this can be expressed by adding a constraint $f_{ij}=0$. As we will discuss later, these observations can be exploited to develop incremental propagators.

\begin{theorem}\label{thm:integrality}
The feasible region associated with the system of linear equations stemming from the definition of \textsc{bin\_counts} is an integral polyhedron.
\end{theorem}
\begin{proof}
We show that the constraint matrix is totally unimodular (TUM); this implies that feasible region associated with the system of linear equations is an integral polyhedron \cite{citeulike:507409}.
In \cite{Heller:Tompkins:1956} the authors proved that, if $A$ is a matrix whose rows can be partitioned into two disjoint sets $B$ and $C$, the following four conditions together are sufficient for A to be totally unimodular:
\begin{enumerate}
\item every column of $A$ contains at most two non-zero entries;
\item every entry in $A$ is $0$, $1$, or $-1$;
\item if two non-zero entries in a column of $A$ have the same sign, then the row of one is in $B$, and the other in $C$;
\item if two non-zero entries in a column of $A$ have opposite signs, then the rows of both are in $B$, or both in $C$.
\end{enumerate}
We reformulate the above set of linear inequalities in standard form as follows
\begin{eqnarray}
&\sum_{j=1}^m f_{ij} 		= 1				&i=1,\ldots,n	\label{eq:4}\\
&\sum_{i=1}^n f_{ij}+s_j 	= \overline{c}_j		&j=1,\ldots,m	\label{eq:5}\\
&\sum_{i=1}^n f_{ij}-e_j	= \underline{c}_j	&j=1,\ldots,m	\label{eq:6}
\end{eqnarray}
where $s_j$ and $e_j$ are nonnegative.
We then consider the coefficient matrix $A$ of this system of linear equations and show that it satisfies the four conditions above. It is possible to show that TUM matrices are closed under the following operations: adding a row or column with at most one non-zero entry, and repeating a row or a column \cite[][lemma 9.2.2]{citeulike:14224552}. We therefore remove from $A$ all columns with one or less non-zero entries (i.e. columns corresponding to slack variables $s_j$ and excess variables $e_j$), as well as repeated rows and columns --- after removing matrix columns associated with slack/excess variables, matrix rows corresponding to constraints \ref{eq:5} and \ref{eq:6} become identical. The resulting constraint matrix is shown in Fig. \ref{tab:constraint_matrix}. If we partition rows as shown in the figure, this will satisfy the four conditions.
\end{proof}

\newcommand\MyLBrace[2]{%
  \left.\rule{0pt}{#1}\right\}\text{#2}}

\begin{figure}
\centering
\resizebox{0.7\textwidth}{!}{%
\begin{tabular}{c@{}l}
\begin{tabular}{lllllllllllll}
$f_{11}$	&$f_{12}$	&\ldots	&$f_{1m}$	&$f_{21}$	&$f_{22}$	&\ldots	&$f_{2m}$&\ldots	&$f_{n1}$		&$f_{n2}$		&\ldots	&$f_{nm}$	\\
\hline
1		&1		&\ldots	&1		&0		&0		&\ldots	&0		&\ldots	&0			&0			&\ldots	&0		\\
0		&0		&\ldots	&0		&1		&1		&\ldots	&1		&\ldots	&0			&0			&\ldots	&0		\\
\vdots	&\vdots	&$\ddots$	&\vdots	&\vdots	&\vdots	&$\ddots$	&\vdots	&$\ddots$	&\vdots		&\vdots		&$\ddots$	&\vdots	\\
0		&0		&\ldots	&0		&0		&0		&\ldots	&0		&\ldots	&1			&1			&\ldots	&1		\\
\hdashline
1		&0		&\ldots	&0		&1		&0		&\ldots	&0		&\ldots	&1			&0			&\ldots	&0		\\
0		&1		&\ldots	&0		&0		&1		&\ldots	&0		&\ldots	&0			&1			&\ldots	&0		\\
\vdots	&\vdots	&$\ddots$	&\vdots	&\vdots	&\vdots	&$\ddots$	&\vdots	&$\ddots$	&0			&\vdots		&$\ddots$	&\vdots	\\
0		&0		&\ldots	&1		&0		&0		&\ldots	&1		&\ldots	&0			&0			&\ldots	&1		
\end{tabular}
&
$\begin{array}{l}
\\
    \MyLBrace{7ex}{$B$} \\ 
    \MyLBrace{7ex}{$C$} 
  \end{array}$
\end{tabular}
}
\caption{TUM constraint matrix $A$}
\label{tab:constraint_matrix}
\end{figure}

\begin{proposition}\label{thm:integraligy}
\textsc{bin\_counts} admits a feasible assignment iff the above system of linear equations admits a solution.
\end{proposition}
\begin{proof}
follows by construction and from Theorem \ref{thm:integrality}.
\end{proof}

\begin{proposition}\label{thm:bc_gac}
BC and GAC are equivalent for \textsc{bin\_counts}.
\end{proposition}
\begin{proof}
the feasible region is a convex integral polyhedron (Theorem \ref{thm:integrality}). 
\end{proof}

More intuitively, holes in $\text{Dom}(c_j)$ do not lead to filtering since each arc $(v^x_i,v^c_j)$ contributes at most one unit of flow; for this reason, we can now relax the assumption that $\text{Dom}(c_j)$ is compact. Domain contraction on $x_i$ leads to filtering only when $x_i$ is instantiated to a value (i.e. $f_{ij}=1$) or when we observe a label wipeout on arc $(v^x_i,v^c_j)$ (i.e. $f_{ij}=0$).

\subsection{Filtering $f_{ij}$ and $c_j$ variables}

GAC on $c_j$ and $x_i$ variable can be enforced by solving $2(n+nm)$ linear programs based upon the above system of linear equations; solving linear programs is in P \cite{citeulike:14183483}.

For all $j=1,\ldots,n$  a lower bound (resp. upper bound) for $\text{Dom}(c_j)$ can be found by solving
\begin{eqnarray}
\text{min (resp. max)}&c_j\\
\text{subject to}\nonumber\\
\text{Eq.}&(\ref{eq:1}),\ldots,(\ref{eq:3})\nonumber
\end{eqnarray}

For all $i=1,\ldots,m$ and $j=1,\ldots,n$ a lower bound (upper bound) on $f_{ij}$ can be found by solving
\begin{eqnarray}
\text{min (resp. max)}&f_{ij}\\
\text{subject to}\nonumber\\
\text{Eq.}&(\ref{eq:1}),\ldots,(\ref{eq:3})\nonumber
\end{eqnarray}
If $f_{ij}=0$, all values in arc $(i,j)$ label should be removed from $\text{Dom}(x_i)$; conversely, if $f_{ij}=1$, all values that do not appear in arc $(i,j)$ label should be removed from $\text{Dom}(x_i)$.

The GAC propagator pseudocode is provided in Algorithm \ref{algo:propagator}. The algorithm takes as input the set $X \equiv \{x_1,\ldots,x_n\}$ of value decision variables, the set $C \equiv \{c_1,\ldots,c_m\}$ of bin counts decision variables, and the set of bin boundaries $B \equiv \{b_1,\ldots,b_{m+1}\}$.
The first loop (line \ref{alg:for_cj}) tightens upper and lower bounds for the variables in $C$.
The second loop (line \ref{alg:for_xij}) tightens upper and lower bounds for the variables in $X$. Both loops use the filtering logic introduced in the previous paragraphs of this section.

\begin{algorithm}
\setstretch{1}
\LinesNumbered
\DontPrintSemicolon 
\KwIn{Sets of value $X \equiv \{x_1,\ldots,x_n\}$ and bin counts $C \equiv \{c_1,\ldots,c_m\}$ decision variables; set of bin boundaries $B \equiv \{b_1,\ldots,b_{m+1}\}$}
\KwOut{Filtered domains for decision variables in $X$ and $C$}
\For{$j \gets 1$ \textbf{to} $m$}{\label{alg:for_cj}
  	$c^{\text{lb}}_j \gets \min c_j$ subject to Eq. (\ref{eq:1}),\ldots,(\ref{eq:3})\;
  	$c^{\text{ub}}_j \gets \max c_j$ subject to Eq. (\ref{eq:1}),\ldots,(\ref{eq:3})\;
  	\eIf{the linear program is infeasible} {
		$\text{Dom}(c_j) \gets \emptyset$\;
	}{
    		$\text{inf}(\text{Dom}(c_j)) \gets c^{\text{lb}}_j$\;
		$\text{sup}(\text{Dom}(c_j)) \gets c^{\text{ub}}_j$\;
  	}
}

\For{$i \gets 1$ \textbf{to} $n$ \textbf{and} $j \gets 1$ \textbf{to} $m$}{\label{alg:for_xij}
	$f^{\text{lb}}_{ij} \gets \min f_{ij}$ subject to Eq. (\ref{eq:1}),\ldots,(\ref{eq:3})\;
  	$f^{\text{ub}}_{ij} \gets \max f_{ij}$ subject to Eq. (\ref{eq:1}),\ldots,(\ref{eq:3})\;
	\uIf{the linear program is infeasible} {
		$\text{Dom}(f_{ij}) \gets \emptyset$\;
	}
	\uElseIf{$f^{\text{lb}}_{ij}=1$} {
		\For{$l \gets 1$ \textbf{to} $m$, $l\neq j$}{
			$\text{Dom}(x_i)= \text{Dom}(x_i) \setminus [b_l,b_{l+1})$\;
		}
	}
	\ElseIf{$f^{\text{ub}}_{ij}=0$} {
		$\text{Dom}(x_i)= \text{Dom}(x_i) \setminus [b_j,b_{j+1})$\;	
	}
}
\caption{\textsc{bincounts} GAC propagator}
\label{algo:propagator}
\end{algorithm}

A pseudocode for incremental propagation triggered by a change in the domain of a variable $v\in X\cup C$ is shown in Algorithm \ref{algo:incremental_propagator}. Incremental propagation requires a set of stored booleans $g_{ij}$, for $i=1,\ldots,n$ and $j=1,\ldots,m$; in CP solvers, stored booleans are boolean variables whose state is tracked during search and restored by backtracking. At the beginning of search, after Algorithm \ref{algo:propagator} has been applied to enforce GAC, each $g_{ij}$ takes value {\em true} if and only if $\text{Dom}(x_i)$ contains at least a value that fits in $[b_j,b_{j+1})$. Algorithm \ref{algo:incremental_propagator} takes then as input the set $X \equiv \{x_1,\ldots,x_n\}$ of value decision variables, the set $C \equiv \{c_1,\ldots,c_m\}$ of bin counts decision variables, the set of bin boundaries $B \equiv \{b_1,\ldots,b_{m+1}\}$, and the stored booleans $g_{ij}$. The propagation logic is divided into two procedures. Let $v$ be the decision variable for which a change in the domain triggered propagation. Procedure \texttt{updateDomainsBin} does not present any substantial difference from the first loop of Algorithm \ref{algo:propagator}. Conversely, procedure \texttt{updateDomainsValue} presents a number of differences from the second loop of Algorithm \ref{algo:propagator}. More specifically, if $v$ belongs to $X$ (line \ref{alg:if_v_is_value}) and $i$ is the index of $v$ in $X$, we record in boolean variable $\hat{g}_{j}$ the current value of stored boolean $g_{ij}$ for all bins $j=1,\ldots,m$. Then, we iterate for all combinations of $i$ and $j$ (line \ref{alg:main_value_loop}). For a given combination of $i$ and $j$, if $v$ belongs to $X$ and $\hat{g}_{j}$ and $g_{ij}$ are not both true (line \ref{alg:skip_value_variables}) --- i.e. domains of variable $x_i$ and variable $v$ do not contain at least a value each that fits in $[b_j,b_{j+1})$ --- we move to the next combination of $i$ and $j$. Otherwise, we update domains similarly to the second loop (line \ref{alg:for_xij}) of Algorithm \ref{algo:propagator} while also updating stored booleans $g_{ij}$.

\begin{algorithm}
\setstretch{1}
\LinesNumbered
\SetAlgoLined\DontPrintSemicolon
  \SetKwFunction{propagate}{propagate}
  \SetKwFunction{updateDomainsBin}{updateDomainsBin}
  \SetKwFunction{updateDomainsValue}{updateDomainsValue}
  \SetKwProg{myalg}{Algorithm}{}{}
  \KwIn{Sets of value $X \equiv \{x_1,\ldots,x_n\}$ and bin counts $C \equiv \{c_1,\ldots,c_m\}$ decision variables; set of bin boundaries $B \equiv \{b_1,\ldots,b_{m+1}\}$; stored booleans $g_{ij}$}
  \myalg{\propagate{v}}{
  \updateDomainsBin{}\;
  \updateDomainsValue{v}\;
  }{}
  \setcounter{AlgoLine}{0}
  \SetKwProg{myproc}{Procedure}{}{}
  \myproc{\updateDomainsBin{}}{
  \For{$j \gets 1$ \textbf{to} $m$}{
  		$c^{\text{lb}}_j \gets \min c_j$ subject to Eq. (\ref{eq:1}),\ldots,(\ref{eq:3})\;
  		$c^{\text{ub}}_j \gets \max c_j$ subject to Eq. (\ref{eq:1}),\ldots,(\ref{eq:3})\;
  		\eIf{the linear program is infeasible} {
  			$\text{Dom}(c_j) \gets \emptyset$\;
  		}{
    			$\text{inf}(\text{Dom}(c_j)) \gets c^{\text{lb}}_j$\;
  			$\text{sup}(\text{Dom}(c_j)) \gets c^{\text{ub}}_j$\;
  		}
	}
  }
  \SetKw{Continue}{continue}
  \setcounter{AlgoLine}{0}  
  \SetKwProg{myproc}{Procedure}{}{}
  \myproc{\updateDomainsValue{v}}{
   \If{$v\in X$}{\label{alg:if_v_is_value}
  		Let $i$ be the index of $v$ in $X$\;
   		\For{$j \gets 1$ \textbf{to} $m$}{
  			$\hat{g}_{j}\gets g_{ij}$\;
  		}
  }
   \For{$i \gets 1$ \textbf{to} $n$ \textbf{and} $j \gets 1$ \textbf{to} $m$}{\label{alg:main_value_loop}
  	\If{$v\in X$ \textbf{and} $\neg(\hat{g}_{j}\wedge g_{ij})$}{\label{alg:skip_value_variables}
  		\Continue
	}
	
  	$f^{\text{lb}}_{ij} \gets \min f_{ij}$ subject to Eq. (\ref{eq:1}),\ldots,(\ref{eq:3})\;
  	$f^{\text{ub}}_{ij} \gets \max f_{ij}$ subject to Eq. (\ref{eq:1}),\ldots,(\ref{eq:3})\;
  	\uIf{the linear program is infeasible} {
  		$\text{Dom}(f_{ij}) \gets \emptyset$\;
	}
  	\uElseIf{$f^{\text{lb}}_{ij}=1$} {
  		\For{$l \gets 1$ \textbf{to} $m$, $l\neq j$}{
  			$\text{Dom}(x_i)= \text{Dom}(x_i) \setminus [b_l,b_{l+1})$\;
  			$g_{il}\gets \textbf{false}$\;
		}
	}
  	\uElseIf{$f^{\text{ub}}_{ij}=0$} {
  		$\text{Dom}(x_i)= \text{Dom}(x_i) \setminus [b_j,b_{j+1})$\;
  		$g_{ij}\gets \textbf{false}$\;	
	}
     }
  }
  \caption{\textsc{bincounts} GAC incremental propagator}
  \label{algo:incremental_propagator}
\end{algorithm}

In the context of an incremental propagator, instead of rebuilding the linear programs from scratch at every propagation round, one may store the current state of the constraint matrix and enforce additional constraints when some $f_{ij}$ are ground. Dual simplex can then be used to find a new feasible and optimal solution. However, in practice the overhead associated with rebuilding the linear programs during propagation is minimal.

{\bf Example.} For the example introduced in Section \ref{sec:decomposition} filtered domains are $\text{Dom}(x_1)=\{3,4\}$, $\text{Dom}(x_2)=\{1,2\}$, $\text{Dom}(x_3)=\{2\}$, $\text{Dom}(c_1)=\{2\}$, and $\text{Dom}(c_2)=\{1\}$. These domains are smaller than those obtained via the decomposition discussed in Section \ref{sec:decomposition}.

It is finally worth mentioning that the filtering algorithm described requires all variables $x_i$ to be assigned a value that can be mapped to one of the available bins. Alternatively, one may want to implement a semantics that simply disregards variables $x_i$ that are assigned a value that cannot be mapped to one of the available bins. This extension of the above propagator is relatively straightforward: one simply needs to add to the graph theoretical construct a ``hidden'' bin to which all values that cannot be mapped to any other bin will be mapped.

\section{Computational study}\label{sec:computational_study}

In this section we contrast performance of the decomposition presented in Section \ref{sec:decomposition} and of the GAC filtering presented in Section \ref{sec:gac}. All experiments in this paper are carried out on a 2.2GHz Intel Core i7 Macbook Air fitted with 8Gb of RAM. 

We implemented the decomposition (Dec) discussed in Section \ref{sec:decomposition} by using two systems featuring implementation of \textsc{global\_cardinality}: JaCoP\footnote{\url{http://jacop.osolpro.com/}} \cite{JaCoP:Guide}, which features a BC implementation based on \cite{citeulike:14184942}; and Choco 3.3.0\footnote{\url{http://www.choco-solver.org/}} \cite{choco}. We developed a \textsc{bin\_counts} global constraint in Choco implementing the GAC filtering in Section \ref{sec:gac}; the linear programming library used in the filtering algorithm is oj! Algorithms.\footnote{\url{http://ojalgo.org/}} 

We consider a set of 50 randomly generated instances. Each instance features $n=15$ variables $v_i$ and $m=9$ bins such that $b_i=5(i-1)$ for $i=1,\ldots,m+1$. $\text{Dom}(v_i)$ comprises up to 10 integer values uniformly distributed in$[0,60)$ --- if a value is generated more than once, the domain will contain less than 10 values. Recall that $c_j$ is a decision variable representing the count of bin $j$, $\text{Dom}(c_j)$ comprises values $\{0,\ldots,U_j\}$, where $U_j$ is a random integer uniformly distributed in $[0,n+1)$. We set as a goal the instantiation of a given fraction $f\in\{0,0.2,0.4,0.6,0.8\}$ of variables $v$, i.e.  $v_1,\ldots,v_{\lfloor n*f\rfloor}$ under a min domain/min value search strategy. All instances generated do admit a solution. In Fig. \ref{fig:filtering} we contrast filtering power for these two approaches. In Fig. \ref{fig:search_time} we contrast search time to achieve the goal. In both figures, we report the 5th and the 95th percentiles, as well as the mean. GAC filters more, but it is slower than the decomposition both in its basic and incremental variants. The decomposition in Choco features a stronger filtering compared to the one in JaCoP.

\begin{figure}
\centering
\resizebox{0.7\columnwidth}{!}{
\includegraphics{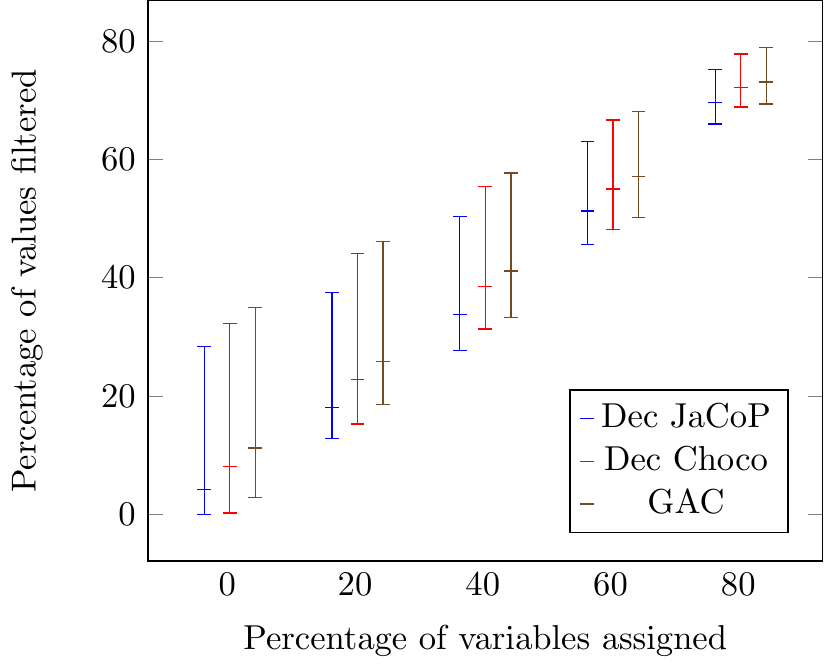}
}
\caption{Relative performance of the decomposition presented in Section \ref{sec:decomposition} and of the GAC propagator presented in Section \ref{sec:gac} in terms of filtering power.}
\label{fig:filtering}
\end{figure}

\begin{figure}
\centering
\resizebox{1\columnwidth}{!}{
\includegraphics{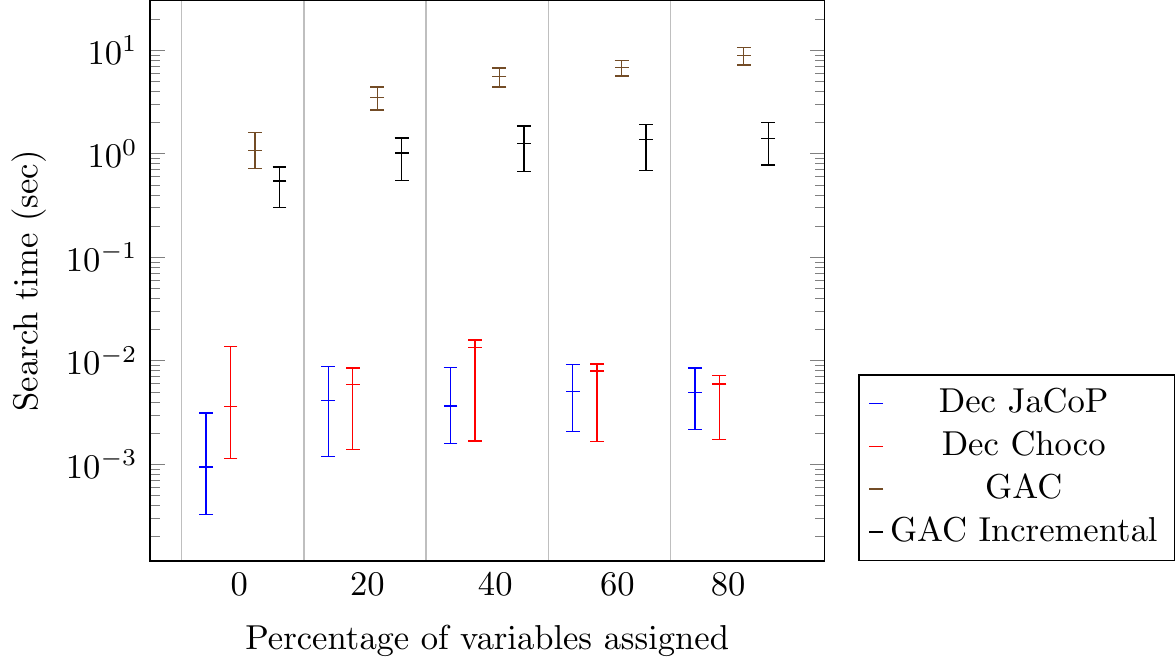}
}
\caption{Relative performance of the decomposition (Dec) presented in Section \ref{sec:decomposition} and of the GAC (basic and incremental) propagator presented in Section \ref{sec:gac} in terms of search time (sec).}
\label{fig:search_time}
\end{figure}

\section{Applications}\label{sec:applications}

In this section we present a number of applications of \textsc{bin\_counts}. In our first application we introduce a new statistical constraint \cite{rossi2014} --- the $\chi^2$ test constraint --- and we develop a decomposition for it which relies on \textsc{bin\_counts}; in our second and third applications, we discuss variants of the Balanced Academic Curriculum Problem \cite{DBLP:journals/corr/cs-PL-0110007} and of the Balanced Nurse Workload Problem \cite{10.2307/822876} which makes use of this new statistical constraint; finally, in our fourth application, we employ a variant of the $\chi^2$ test constraint in the context of a well-known problem from the literature on confidence interval analysis.

Since all models presented feature a mix of continuous and real variables, we rely on the extensions discussed in \cite{tricks2013}, which makes it possible to model real variables and constraints within a Choco model and delegate associated reasoning during search to Ibex,\footnote{\url{http://www.ibex-lib.org/}} a library for constraint processing over real numbers. 

\subsection{Pearson's $\chi^2$ test statistical constraint}

Statistical constraints were originally discussed in \cite{rossi2014}. A statistical constraint is a constraint that embeds a parametric or a non-parametric statistical model and a statistical test with significance level\footnote{The significance level is the probability of rejecting a null hypothesis, given that it is true.} $\alpha$ that is used to determine which assignments satisfy the constraint. Recent applications include \cite{Pachet:2015:GNS:2832581.2832596,sony2016}.

Existing statistical constraints include the Student's $t$ test constraint and the Kolmogorov-Smirnov constraint, both of which rely on the associated  statistical tests. Although no implementation exists for the Student's $t$ test constraint, if one exploits Choco extensions for modeling real variables, a decomposition is immediate since the associated test statistic can be easily modelled as an algebraic expression and forced to be less or equal to a given critical value. 

A statistical constraint not yet discussed in the literature is the $\chi^2$ test constraint. This constraint relies on Pearson's $\chi^2$ test of goodness of fit \cite{doi:10.1080/14786440009463897}, which establishes whether an observed frequency distribution differs from a theoretical distribution. In the $\chi^2$ test statistical constraint
\[\chi^2\text{-test}^\alpha_{b_1,\ldots,b_{m+1}}(v_1,\ldots,v_n;c_1,\ldots,c_m;t_1,\ldots,t_m)\]
$v_i$ is a decision variable that represents a random variate; $c_j$ is a decision variable that represents the number of variables $v_i$ which take a value in $[b_j,b_{j+1})$; $t_j$ is a decision variable that represents the theoretical reference count for bin $[b_j,b_{j+1})$. An assignment satisfies $\chi^2$ test statistical constraint iif a $\chi^2$ test at significance level $\alpha$ fails to reject the null hypothesis that the observed difference between the theoretical reference counts and the observed counts arose by chance.

A decomposition for the $\chi^2$ test statistical constraint is shown in Fig. \ref{model:chi_square}: constraint (1) deals with the computation of the counts, constraint (2) restricts the $\chi^2$ test statistic.

For the $\chi^2$ test statistical constraint it is worth observing that an $\alpha$ close to one restricts the solution space to those assignments for which the observed difference between the theoretical reference counts and the observed counts is small. As $\alpha$ decreases, higher fluctuations from the theoretical reference counts will be tolerated.

\begin{figure}
\begin{center}
        \framebox{
        \begin{tabular}{ll}
        \mbox{\bf Constraints:}\\
        \multicolumn{2}{l}{
        ~~~~(1)~~$\textsc{bin\_counts}_{b_1,\ldots,b_{m+1}}(v_1,\ldots,v_n;c_1,\ldots,c_m)$}\\
        \multicolumn{2}{l}{
        ~~~~(2)~~$\sum_{i=1}^m (c_i - t_i)^2/t_i\leq F^{-1}_{\chi_{m-1}^2}(1-\alpha)$}\\  
        \\     
        \mbox{\bf Parameters:} \\
        ~~~~$b_1,\ldots,b_{m+1}$ 			&bin boundaries\\
        ~~~~$F^{-1}_{\chi_{m-1}^2}$ 			&inverse $\chi^2$ distribution with\\ 
        					   				&$m-1$ degrees of freedom\\
        ~~~~$\alpha$ 						&target significance for the $\chi^2$ test\\
        \\
        \mbox{\bf Decision variables:} \\
        ~~~~$v_1,\ldots,v_n$ 				&observed values\\
        ~~~~$c_1,\ldots,c_m$ 				&bin counts\\
        ~~~~$t_1,\ldots,t_m$ 				&target counts for each bin
        \end{tabular}
        }
    \caption{$\chi^2$ test statistical constraint decomposition}
    \label{model:chi_square}
\end{center}    
\end{figure}

{\bf Example.} We consider a problem with $n=24$ variables $v_i$ and $m=6$ bins such that $b_i=5(i-1)$ for $i=1,\ldots,m+1$. $\text{Dom}(v_i)$ comprises values $\{0,\ldots,U^x_i\}$, where $U^x_i$ is a random integer number uniformly distributed in $[0,30)$. $\text{Dom}(c_j)$ comprises values $\{0,\ldots,n\}$. We implemented the model and set as a goal the instantiation of variables $v_1,\ldots,v_n$. Theoretical reference counts $t_j$ for the $m$ bins are $t=\{2,4,10,4,2,2\}$. The bin counts for two solutions obtained for significance levels $\alpha\in\{0.95,0.99\}$ are shown in Fig. \ref{fig:bincounts_chi} and contrasted against theoretical reference counts. 

\begin{figure}
\centering
\includegraphics{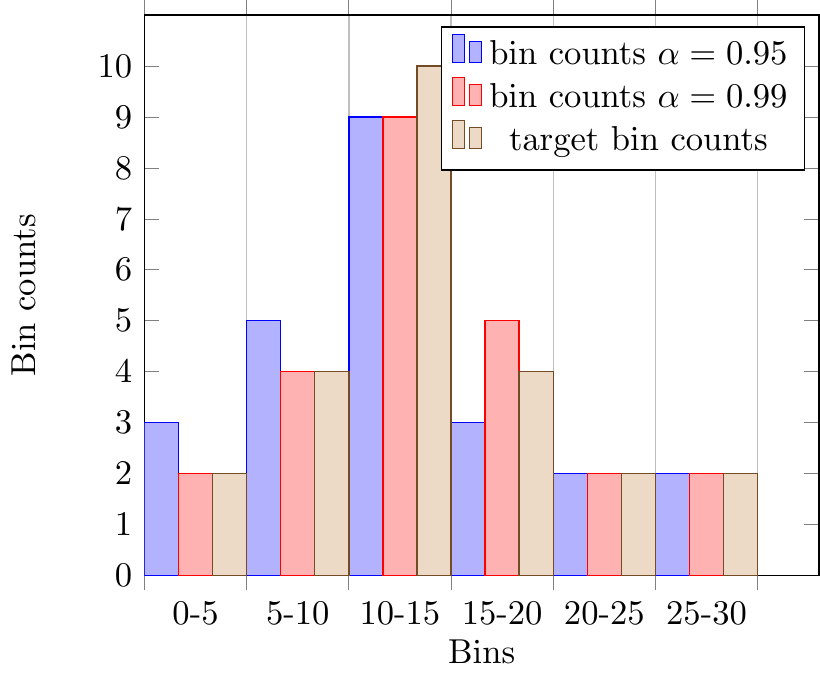}
\caption{Sample bincounts obtained in the context of our numerical example for different values of $\alpha$: when $\alpha=0.95$, $\chi^2= 1.10 \leq F^{-1}_{\chi_{m-1}^2}(1-\alpha)=1.14$; when $\alpha=0.99$, $\chi^2= 0.35 \leq F^{-1}_{\chi_{m-1}^2}(1-\alpha)=0.55$. 
}
\label{fig:bincounts_chi}
\end{figure}

\subsection{Balanced academic curriculum problem}

CSPLib \cite{citeulike:14181502} problem 30, the Balanced Academic Curriculum Problem (BACP)
\cite{DBLP:journals/corr/cs-PL-0110007}, asks to assign courses to semesters in such a way as to balance academic load --- the sum of credits from courses in a semester --- among semesters. In addition, there are constraints on minimum and maximum number of courses per semester, and some courses are prerequisite to others, e.g. course $B$ must be assigned to a semester following the one in which course $A$ is assigned. The objective that is optimised in order to achieve a balanced curriculum varies from work to work: the original formulation seeks to minimise the maximum load in any given semester, other formulations, see e.g. \cite{citeulike:14181514}, employ the $L_2$-deviation to measure balance. In this section we discuss a different strategy, based on 
the $\chi^2$ test statistic, to measure curriculum compliance with a target credit load distribution among semesters.

\begin{figure}[h!]
\begin{center}
        \framebox{
        \begin{tabular}{ll}
        \mbox{\bf Constraints:}\\
        \multicolumn{2}{l}{
        ~~~~(1)~~$\textsc{global\_cardinality}_{1,\ldots,S}(s;c)$}\\
        \multicolumn{2}{l}{
        ~~~~(2)~~$\textsc{bin\_packing}_w(s;l)$}\\      
        \multicolumn{2}{l}{
        ~~~~(3)~~$s_i<s_j$ (course prerequisites)}\\
        \multicolumn{2}{l}{
        ~~~~(4)~~$\textsc{bin\_counts}_{b_1,\ldots,b_{m+1}}(l;o)$}\\
        \multicolumn{2}{l}{        
        ~~~~(5)~~$\sum_{k=1}^m (o_k -t_k)^2/t_k\leq F^{-1}_{\chi_{m-1}^2}(1-\alpha)$}\\  
        \\
        \mbox{\bf Parameters:} \\
        ~~~~$b_1,\ldots,b_{m+1}$ 	&bin boundaries\\
        ~~~~$S$ 					&number of semesters\\
        ~~~~$w_i$ 				&course $i$ credits\\
        ~~~~$t_k$					&target occurrences in bin $k$\\    
        \\       
        \mbox{\bf Decision variables:} \\
        ~~~~$s_i$ 				&course $i$ semester\\
        ~~~~$l_j$ 					&semester $j$ load\\
        ~~~~$c_j$ 				&number of courses in semester $j$\\
        ~~~~$o_k$ 				&load occurrences in bin $k$
        \end{tabular}
        }
    \caption{A CP formulation for the BACP}
    \label{model:bacp}
\end{center}    
\end{figure}

As discussed in \cite{citeulike:14181514}, a CP formulation for the BACP (Fig. \ref{model:bacp}) includes one decision variable $s_i$ per course $i$, which indicates to which semester the course is assigned; one decision variables $l_j$ per semester $j$ recording its academic load; and one decision variable $c_j$ per semester recording the number of courses allocated to it. On these variables we enforce a \textsc{global\_cardinality} constraint (1) that links $s_i$ and $c_j$ to record the number of courses in each semester; a \textsc{bin\_packing} constraint (2) that links $s_i$ and $l_j$ to capture course credits in each semester; and binary inequalities (3) between pairs of $s_i$ variables to record course prerequisites.

In addition to the standard constructs above, our formulation includes a list of scalar bin boundaries $b_1,\ldots,b_k,\ldots,b_{m+1}$ and associated scalar target occurrences $t_k$, for $k=1,\ldots,m$, which are employed to capture the ideal semester load distribution we aim for. The load occurrences in bin $k$ are recorded by one decision variable $o_k$ per bin via a \textsc{bin\_counts} constraint (4) that links variables $l_j$ and $o_k$. Finally, the associated $\chi^2$ test statistic is forced to be less or equal to a relevant critical value (5). 

Our formulation leads to a constraint satisfaction and not to a constraint optimisation problem. Alternatively, one may decide to model target occurrences $t_k$ as decision variables and minimise or maximise a given measure associated with the load distribution, e.g. mean, variance, loss, etc; therefore our study can be seen as a generalisation of existing studies that investigated the optimisation of specific moments of a distribution.

The traditional smallest-domain-first (variable) and semester with the smallest academic load (value) criteria \cite{Schaus2009} is not suitable for our problem, since our goal is not do minimise the maximum load or the deviation from the mean load. Our aim is chiefly to demonstrate the effectiveness of GAC propagation, therefore we adopted a simple min domain/min value search strategy on variables $s_i$. We also implemented the symmetry breaking strategy in \cite{symcon2007}.

We considered the 28 instances in \cite{citeulike:14181502}. All instances feature 50 courses, 10 semesters, a min of 2 credits and a max of 100 credits per semester; a min of 2 courses and a max of 10 courses per semester. Let $L_{\text{ub}}$ be the load per period upper bound, we consider bin bounds $\{0,15,20,30,35,L_{\text{ub}}+1\}$ and associated target occurrences $\{1,2,4,2,1\}$; significance level $\alpha$ is set to 0.99.

{\bf Example.} We consider the instance ``bacp-1'' from the CSPLib testbed. A feasible schedule is shown in Fig. \ref{fig:bacp-1_plan}; semester loads are shown in Table \ref{tab:semester_load}.

\begin{figure}[h!]
\centering
\includegraphics[width=0.9\columnwidth]{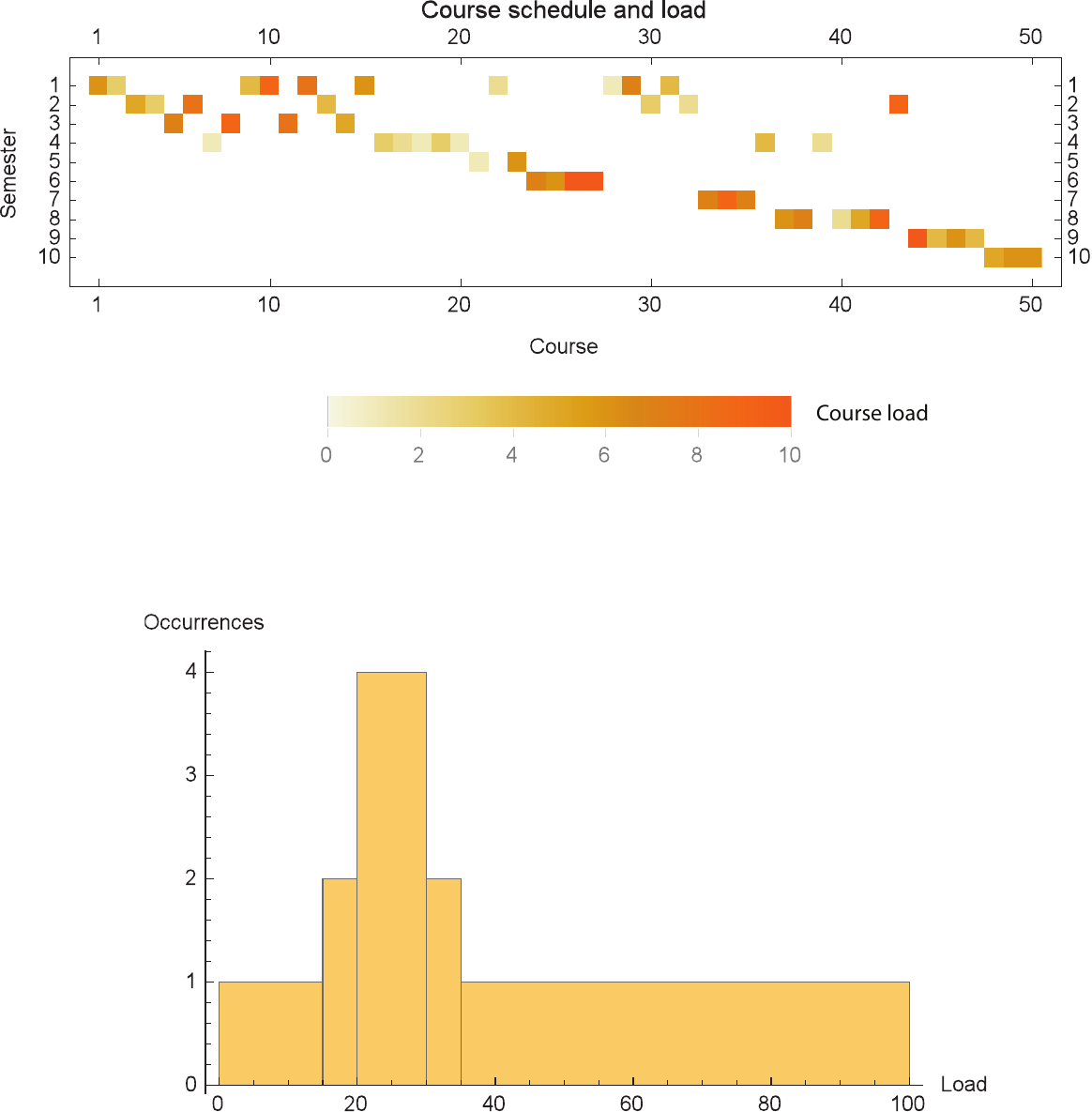}
\caption{Feasible course schedule for instance ``bacp-1'' from the CSPLib testbed.}
\label{fig:bacp-1_plan}
\end{figure}

\begin{table}[h!]
\centering
\begin{tabular}{l|rrrrrrrrrr}
Semester	&1&2&3&4&5&6&7&8&9&10\\
\hline
Load		&50& 34& 29& 17& 7& 33& 23& 29& 24& 17
\end{tabular}
\caption{Semester loads of the feasible course schedule (Fig. \ref{fig:bacp-1_plan}) for instance ``bacp-1'' from the CSPLib testbed.}
\label{tab:semester_load}
\end{table}

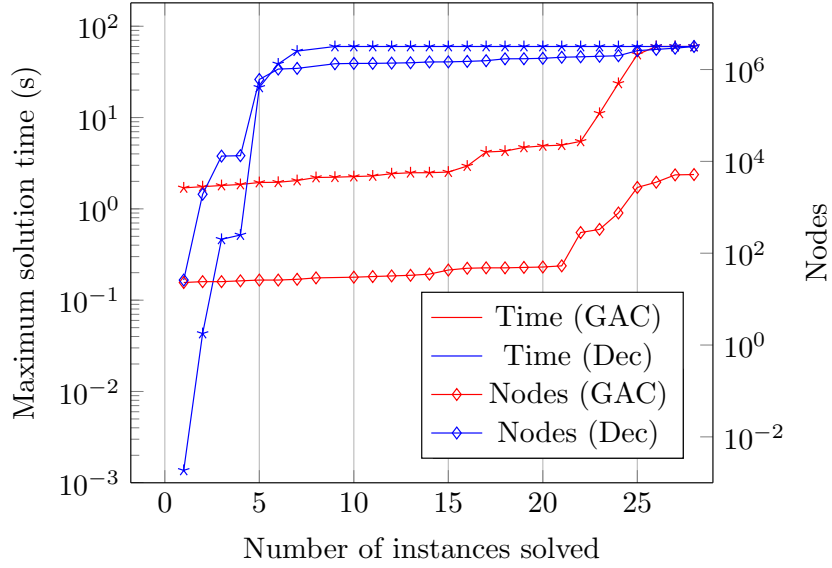
\begin{figure}[h!]
\centering
\resizebox{0.75\columnwidth}{!}{
\begin{tikzpicture}
	\begin{axis}[
		xmax=29,
		ymin=10e-4,
		xmajorgrids,
		legend style={at={(0.95,0.05)},anchor=south east},
	        axis y line*=left,
		ymode=log,
		xlabel=Number of instances solved,
		ylabel=Maximum solution time (s)]
	\addplot[color=red,mark=star] coordinates {
(	1	,	1.6974802	)
(	2	,	1.7501208	)
(	3	,	1.8016485	)
(	4	,	1.8497909	)
(	5	,	1.9436326	)
(	6	,	1.9554434	)
(	7	,	2.04248	)
(	8	,	2.2048786	)
(	9	,	2.2217205	)
(	10	,	2.2553117	)
(	11	,	2.2930627	)
(	12	,	2.426455	)
(	13	,	2.4863276	)
(	14	,	2.4880068	)
(	15	,	2.5291321	)
(	16	,	2.943285	)
(	17	,	4.1922665	)
(	18	,	4.2866864	)
(	19	,	4.71163	)
(	20	,	4.893437	)
(	21	,	4.987908	)
(	22	,	5.485529	)
(	23	,	11.128678	)
(	24	,	23.669256	)
(	25	,	49.035015	)
(	26	,	60.007404	)
(	27	,	60.011955	)
(	28	,	60.017715	)
	}; \label{plot_time_gac}
	\addlegendentry{Time (GAC)}	
	\addplot[color=blue,mark=star] coordinates {
(	1	,	0.001367503	)
(	2	,	0.043138888	)
(	3	,	0.46363023	)
(	4	,	0.5170676	)
(	5	,	21.399738	)
(	6	,	38.651035	)
(	7	,	53.42189	)
(	8		60.000015	)
(	9	,	60.00002	)
(	10	,	60.000023	)
(	11	,	60.000027	)
(	12	,	60.000034	)
(	13	,	60.000034	)
(	14	,	60.00004	)
(	15	,	60.00004	)
(	16	,	60.00004	)
(	17	,	60.000042	)
(	18	,	60.000042	)
(	19	,	60.000046	)
(	20	,	60.000053	)
(	21	,	60.00006	)
(	22	,	60.00006	)
(	23	,	60.00007	)
(	24	,	60.000088	)
(	25	,	60.000095	)
(	26	,	60.000107	)
(	27	,	60.000134	)
(	28	,	60.00016	)
	}; \label{plot_time_dec}
	\addlegendentry{Time (Dec)}		
	\end{axis}
	
	\begin{axis}[
		xmax=29,
		ymin=10e-4,
		legend style={at={(0.95,0.05)},anchor=south east},
	        axis y line*=right,
	        ylabel near ticks, 
	        yticklabel pos=right,
	        axis x line=none,
		ymode=log,
		ylabel=Nodes]
	\addlegendimage{/pgfplots/refstyle=plot_time_gac}\addlegendentry{Time (GAC)}
	\addlegendimage{/pgfplots/refstyle=plot_time_dec}\addlegendentry{Time (Dec)}
	\addplot[color=red,mark=diamond] coordinates {
(	1	,	23	)
(	2	,	24	)
(	3	,	24	)
(	4	,	25	)
(	5	,	26	)
(	6	,	26	)
(	7	,	27	)
(	8	,	29	)
(	9		30	)
(	10	,	30	)
(	11	,	31	)
(	12	,	32	)
(	13	,	33	)
(	14	,	35	)
(	15	,	43	)
(	16	,	47	)
(	17	,	48	)
(	18	,	48	)
(	19	,	49	)
(	20	,	50	)
(	21	,	53	)
(	22	,	282	)
(	23	,	330	)
(	24	,	752	)
(	25	,	2705	)
(	26	,	3516	)
(	27	,	5078	)
(	28	,	5166	)
	}; \label{plot_node_gac}
	\addlegendentry{Nodes (GAC)}	
	\addplot[color=blue,mark=diamond] coordinates {
(	1	,	26	)
(	2	,	1924	)
(	3	,	13026	)
(	4	,	13294	)
(	5	,	606914	)
(	6	,	1023393	)
(	7	,	1058771	)
(	8		1233711	)
(	9	,	1337324	)
(	10	,	1358165	)
(	11	,	1361137	)
(	12	,	1380860	)
(	13	,	1403910	)
(	14	,	1461458	)
(	15	,	1468006	)
(	16	,	1503413	)
(	17	,	1557695	)
(	18	,	1712113	)
(	19	,	1719878	)
(	20	,	1754997	)
(	21	,	1850144	)
(	22	,	1892450	)
(	23	,	1944231	)
(	24	,	1998084	)
(	25	,	2574795	)
(	26	,	2743705	)
(	27	,	2917854	)
(	28	,	3177022	)
	}; \label{plot_node_dec}
	\addlegendentry{Nodes (Dec)}
	\end{axis}
\end{tikzpicture}
}
\caption{Solution times for the BACP test bed}
\label{fig:bacp_time}
\end{figure}

Fig. \ref{fig:bacp_time} shows the maximum solution time and the maximum number of nodes observed for any given number of instances. Because of the high value of $\alpha$, for all instances solved the observed $\chi^2$ statistic is zero, i.e no deviation from target occurrences; in the next sections we will investigate cases in which a lower value of $\alpha$, and thus larger deviations, are both relevant and desirable. For the GAC model the average solution time is 11.6s; half of the instances are solved in less than 2.5s; the average number of nodes explored is 662; three instances (bacp-22,25,27) could not be solved within the given time limit of 60s.
The Dec model could only solve 7 instances within the given time limit; the average solution time is 49.1s; the average number of nodes explored is 1466725. GAC therefore brings orders of magnitude improvements in both nodes explored and search time.

\subsection{Balanced nursing workload problem}

We further investigate application of \textsc{bin\_count} in the context of CSPLib \cite{citeulike:14181502} problem 69, the Balanced Nursing Workload Problem (BNWP), which was originally introduced in \cite{10.2307/822876}. The aim in this problem is to design a balanced workload for nurses caring newborn patients requiring different amount of care ({\em acuity}). Patients belong to a zone and a nurse can only work in a single zone. There are lower and upper limits on the number of patient the nurse can handle and on the associated workload expressed in terms of total acuity.

The authors in \cite{10.2307/822876} proposed a Mixed Integer Programming approach. CP approaches based on the \textsc{spread} constraint were discussed in \cite{Schaus:2009:SLB:1560579.1560600,citeulike:14187845}, where the authors proposed a decomposition strategy that pre-computes the number of nurses for each zone and then solves each zone separately. An alternative CP approach based on the \textsc{dispersion} constraint was discussed in \cite{citeulike:14181514} and more recently in \cite{citeulike:14187844}, where nurse dependent workloads are modelled.

It is out of the scope of this section to provide a comprehensive discussion on the BNWP. Our key concern here is to provide an alternative reformulation for the problem based on the \textsc{bin\_count} constraint. While most previous works on this problem focused on minimizing the $L_2$-norm. We suggest an alternative approach in which the decision markers assigns a certain number of ``slots'' to each nurse and then sets an  ``ideal'' workload distribution over these slots. For instance, each nurse should be dealing with up to 6 patients of which ideally, 2 should have acuity in [0,30), 2 in [30,60), and 2 in [60,100). Fluctuations from this ideal patient distribution are accepted, but should be minimised for the nurse population.

While modeling our variant of the problem we deviate from the standard CP formulation discussed in \cite{citeulike:14181514}. More specifically, in our model (Fig. \ref{model:bnwp}) there are $N$ nurses, $P$ patients, and $S$ patient slots per nurse. The acuity of patient $p$ is $a_p$. The model 
features $S\cdot N$ decision variables $g^n_s$, each of which represents patient allocated to slot $s$ of nurse $n$. The acuity condition of $g^n_s$ is represented by decision variable $c^n_s$. Acuity occurrences in bin $k$ of nurse $n$ are represented by decision variable $o^n_k$. There are $b_1,\ldots,b_{m+1}$ acuity bin boundaries and, as mentioned, the decision maker must set a target number of patients in bin $k$ for a nurse. Note that it is possible to express nurse dependent target occurrence distributions; but to keep the discussion simple, we will here assume all nurses share the same target occurrence distribution. 
On these variables and parameters we enforce the following constraints: an \textsc{all\_different} constraint (2) on variables $g^n_s$, to ensure all patients are assigned to different nurses' slots --- note that to ensure that the number of patients is a multiple of $S\cdot N$ it is possible to insert a number of dummy patients with zero acuity; an \textsc{element} constraint (3) to ensure that variable $c^n_s$ represents the acuity of patient $g^n_s$; a \textsc{bin\_counts} constraint (4) for each nurse to relate variables variable $c^n_s$ with acuity occurrences $o^n_k$; and a linear inequality (5) that ensures the normalised deviation from target occurrences (effectively a $\chi^2$ statistic) remains below $K$ for all nurses. $K$ is finally minimised in the objective (1).

\begin{figure}[h!]
\begin{center}
        \framebox{
        \begin{tabular}{ll}
	\mbox{\bf Objective function:}\\
	\multicolumn{2}{l}{
	~~~~(1)~~$\min$ $K$}\\
	\\
        \mbox{\bf Constraints:}\\
        \multicolumn{2}{l}{
        ~~~~(2)~~$\textsc{all\_different}(g)$}\\
        \multicolumn{2}{l}{
        ~~~~(3)~~$\textsc{element}(c^n_s,a,g^n_s)$ for all $n$ and $s$}\\
        \multicolumn{2}{l}{
        ~~~~(4)~~$\textsc{bin\_counts}_{b_1,\ldots,b_{m+1}}(c^n;o)$ for all $n$}\\
        \multicolumn{2}{l}{        
        ~~~~(5)~~$\sum_{k=1}^m (o^n_k -t_k)^2/t_k\leq K$ for all $n$}\\ 
        \\
        \mbox{\bf Parameters:} \\
        ~~~~$N$ 					&number of nurses, index $n$\\
        ~~~~$P$ 					&number of patients, index $p$\\   
        ~~~~$S$ 					&number of patient slots per nurse, index $s$\\
        ~~~~$a_p$ 			 	&acuity of patient $p$\\     
        ~~~~$b_1,\ldots,b_{m+1}$ 	&acuity bin boundaries\\        
        ~~~~$t_k$					&target \# patients in bin $k$ for a nurse\\    
        \\       
        \mbox{\bf Decision variables:} \\
        ~~~~$g^n_s$ 				&patient allocated to slot $s$ of nurse $n$\\
        ~~~~$c^n_s$ 				&acuity condition of $g^n_s$\\
        ~~~~$o^n_k$ 				&acuity occurrences in bin $k$ of nurse $n$
        \end{tabular}
        }
    \caption{A CP formulation for the BNWP}
    \label{model:bnwp}
\end{center}    
\end{figure}

By using the model introduced, we solved slightly modified versions of the test instances originally proposed in \cite{Schaus:2009:SLB:1560579.1560600} and available on CSPLib. More specifically, we did adopt a decomposition approach and we thus solved each zone separately. Most importantly, we set the number of slot per nurse to 6 and adopted the previously mentioned ``ideal'' workload distribution, which assigns for each nurse 2 patients to each of the 3 bins identified by bin boundaries $b_1=0,b_2=30,b_3=60,b_4=100$. We precomputed the number of nurses needed for patients in any given zone, i.e. $N=\lceil P/S\rceil$, and we extended the list of patients to include $P'$ patients, with $Z=P'-P$ zero-acuity patients to ensure that $P'=N\cdot S$. We ignored the min/max number of patients per nurse as well as the maximum workload, since these instance parameters become irrelevant once an ``ideal'' workload distribution is defined for the nurses.

{\bf Example.} We consider instance ``2zones0'' available on CSPLib. In this instance, we have two zones and 11 nurses. To keep the example simple, we decompose the problem and focus solely on zone 1. There are $P=17$ patients with acuities \[a=\{59,57,50,44,42,40,39,39,33,33,32,27,26,22,20,17,11\}.\] We assume each nurse should care of $S=6$ patients in total therefore we can cover the zone by using $N=3$ nurses in total; a dummy patients (\#18) with zero acuity must then be added to ensure that $P'=N\cdot S$. The allocation plan that minimises $K$ is  shown in Table \ref{tab:allocation_2zones0}; the observed patient acuity occurrence distributions are shown in Table \ref{tab:occurrences_distribution_2zones0}. The minimised maximum $\chi^2$ statistic value of the optimal plan is 4. It is clear that balancing allocations in such a way as to minimise fluctuations from an ideal workload distribution is not a trivial task even for an instance as small as this one.

\begin{table}
\centering
\resizebox{\columnwidth}{!}{
\begin{tabular}{lrrrrrr|rrrrrr|rrrrrr}
		&\multicolumn{6}{c}{Nurse 1}	&\multicolumn{6}{c}{Nurse 2}	&\multicolumn{6}{c}{Nurse 3}\\
s		&1	&2	&3	&4	&5	&6	&1	&2	&3	&4	&5	&6	&1	&2	&3	&4	&5	&6\\
\hline
$g^n_s$	&1	&4	&7	&11	&13	&18	&2	&5	&8	&12	&14	&17	&3	&6	&9	&10	&15	&16\\
$c^n_s$	&59	&44	&39	&32	&26	&0	&57	&42	&39	&27	&22	&11	&50	&40	&33	&33	&20	&17\\
\end{tabular}
}
\caption{Optimal nurse-patient allocation for instance ``2zones0''}
\label{tab:allocation_2zones0}
\end{table}
 
\begin{table}
\centering
\resizebox{0.5\columnwidth}{!}{
\begin{tabular}{l|ccc|c}
\multicolumn{1}{c}{}&\multicolumn{3}{c}{Bins}\\
 		&[0,30)	&[30,60)	&[60,100)	&$K$ ($\chi^2$ statistic)\\
\hline		
 Nurse 1	&2		&4		&0		&4\\
 Nurse 2	&3		&3		&0		&3\\
 Nurse 3	&2		&4		&0		&4\\
 \hline
 Target	&2		&2		&2		
 \end{tabular}
 }
\caption{Observed patient acuity occurrence distributions for instance ``2zones0''}
\label{tab:occurrences_distribution_2zones0}
\end{table}

Once more, we adopted a simple min domain/min value search strategy in which our goal is instantiation of variables $g^n_s$. We also implemented a naive symmetry breaking strategy that forces $g^n_s<g^n_{s+1}$ for all $s$ and $n$; and also $g^n_1<g^{n+1}_1$ for all $n$. Results of our computational study are shown in Fig. \ref{fig:bacp_time}, which as before illustrates the maximum solution time and maximum number of nodes observed for any given number of instances. Because of the zone-based decomposition we adopted, we eventually solved a total of 91 instances. Similarly to what we have done for the BACP, we here investigate differences between the GAC and the decomposition approach for the \textsc{bin\_count}. Surprisingly, the decomposition approach is generally one order of magnitude faster for this problem. However, by observing the number of nodes explored, it is clear that the GAC approach still leads to more effective pruning.

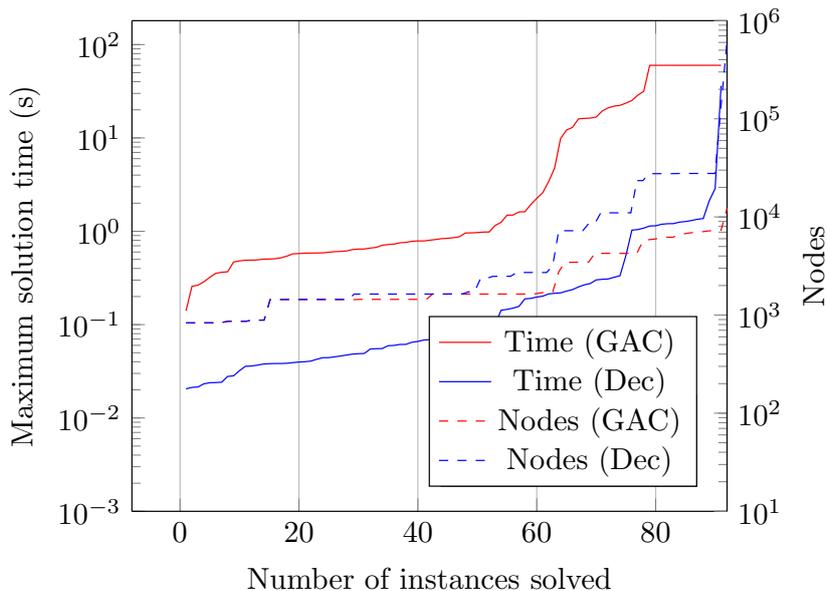
\begin{figure}[h!]
\centering
\resizebox{0.75\columnwidth}{!}{
\begin{tikzpicture}
	\begin{axis}[
		xmax=92,
		ymin=10e-4,
		xmajorgrids,
		legend style={at={(0.95,0.05)},anchor=south east},
	        axis y line*=left,
		ymode=log,
		xlabel=Number of instances solved,
		ylabel=Maximum solution time (s)]
	\addplot[color=red] coordinates {
(	1	,	0.14069042	)
(	2	,	0.25747246	)
(	3	,	0.264466	)
(	4	,	0.28751147	)
(	5	,	0.32148415	)
(	6	,	0.3545526	)
(	7	,	0.3637679	)
(	8	,	0.3682336	)
(	9	,	0.46844783	)
(	10	,	0.48150668	)
(	11	,	0.4901053	)
(	12	,	0.49075693	)
(	13	,	0.49337795	)
(	14	,	0.5029566	)
(	15	,	0.5040281	)
(	16	,	0.5082681	)
(	17	,	0.5203837	)
(	18	,	0.5427194	)
(	19	,	0.57397497	)
(	20	,	0.5776166	)
(	21	,	0.5827136	)
(	22	,	0.5847328	)
(	23	,	0.585159	)
(	24	,	0.5876289	)
(	25	,	0.59604913	)
(	26	,	0.6054454	)
(	27	,	0.6092299	)
(	28	,	0.6158583	)
(	29	,	0.63722014	)
(	30	,	0.6443466	)
(	31	,	0.6464129	)
(	32	,	0.66126376	)
(	33	,	0.67217815	)
(	34	,	0.70805156	)
(	35	,	0.72102463	)
(	36	,	0.7297309	)
(	37	,	0.7536947	)
(	38	,	0.76037395	)
(	39	,	0.78157246	)
(	40	,	0.78817284	)
(	41	,	0.7884552	)
(	42	,	0.8024213	)
(	43	,	0.8166299	)
(	44	,	0.83384395	)
(	45	,	0.8418302	)
(	46	,	0.8538491	)
(	47	,	0.87472117	)
(	48	,	0.9593871	)
(	49	,	0.9618546	)
(	50	,	0.9697458	)
(	51	,	0.9801059	)
(	52	,	0.9833705	)
(	53	,	1.1533707	)
(	54	,	1.2482382	)
(	55	,	1.487984	)
(	56	,	1.4918275	)
(	57	,	1.6118771	)
(	58	,	1.6199079	)
(	59	,	1.9362205	)
(	60	,	2.2554204	)
(	61	,	2.5889864	)
(	62	,	3.4141178	)
(	63	,	4.763428	)
(	64	,	9.853514	)
(	65	,	12.151078	)
(	66	,	12.955053	)
(	67	,	15.976505	)
(	68	,	16.20936	)
(	69	,	16.299015	)
(	70	,	16.762503	)
(	71	,	19.426325	)
(	72	,	20.996134	)
(	73	,	21.924784	)
(	74	,	22.336847	)
(	75	,	23.610006	)
(	76	,	25.168224	)
(	77	,	28.738623	)
(	78	,	31.637419	)
(	79	,	60.001686	)
(	80	,	60.002625	)
(	81	,	60.002953	)
(	82	,	60.004078	)
(	83	,	60.00419	)
(	84	,	60.00524	)
(	85	,	60.009342	)
(	86	,	60.010857	)
(	87	,	60.012787	)
(	88	,	60.01366	)
(	89	,	60.01467	)
(	90	,	60.019066	)
(	91	,	60.03091	)
	}; \label{plot_time_gac}
	\addlegendentry{Time (GAC)}	
	\addplot[color=blue] coordinates {
(	1	,	0.020495992	)
(	2	,	0.021264177	)
(	3	,	0.021461053	)
(	4	,	0.02320561	)
(	5	,	0.023870775	)
(	6	,	0.023952588	)
(	7	,	0.024153652	)
(	8	,	0.027841993	)
(	9	,	0.028290095	)
(	10	,	0.03214114	)
(	11	,	0.035818093	)
(	12	,	0.035994593	)
(	13	,	0.036788262	)
(	14	,	0.037781112	)
(	15	,	0.038160287	)
(	16	,	0.03832796	)
(	17	,	0.038331226	)
(	18	,	0.038586672	)
(	19	,	0.039180845	)
(	20	,	0.039750792	)
(	21	,	0.04004505	)
(	22	,	0.040667716	)
(	23	,	0.04256006	)
(	24	,	0.044192825	)
(	25	,	0.04429172	)
(	26	,	0.04518634	)
(	27	,	0.046110544	)
(	28	,	0.04723242	)
(	29	,	0.048311964	)
(	30	,	0.0487951	)
(	31	,	0.049141966	)
(	32	,	0.05480948	)
(	33	,	0.05521231	)
(	34	,	0.055430178	)
(	35	,	0.059427623	)
(	36	,	0.059889473	)
(	37	,	0.061354533	)
(	38	,	0.06145344	)
(	39	,	0.06474019	)
(	40	,	0.06605123	)
(	41	,	0.06838946	)
(	42	,	0.06909571	)
(	43	,	0.07150108	)
(	44	,	0.07388619	)
(	45	,	0.07455308	)
(	46	,	0.075107686	)
(	47	,	0.076913625	)
(	48	,	0.0805693	)
(	49	,	0.0831426	)
(	50	,	0.083744936	)
(	51	,	0.0889181	)
(	52	,	0.09225724	)
(	53	,	0.09946675	)
(	54	,	0.14269117	)
(	55	,	0.14568517	)
(	56	,	0.15023641	)
(	57	,	0.15680176	)
(	58	,	0.18873954	)
(	59	,	0.1909599	)
(	60	,	0.19755954	)
(	61	,	0.20187283	)
(	62	,	0.21334337	)
(	63	,	0.21686019	)
(	64	,	0.21884938	)
(	65	,	0.2285962	)
(	66	,	0.23711847	)
(	67	,	0.25321808	)
(	68	,	0.26692557	)
(	69	,	0.27559167	)
(	70	,	0.30155697	)
(	71	,	0.3060844	)
(	72	,	0.30831328	)
(	73	,	0.32152835	)
(	74	,	0.33442804	)
(	75	,	0.55197716	)
(	76	,	1.0278946	)
(	77	,	1.0444493	)
(	78	,	1.0816026	)
(	79	,	1.1401501	)
(	80	,	1.1447217	)
(	81	,	1.189675	)
(	82	,	1.2052017	)
(	83	,	1.2089581	)
(	84	,	1.2529713	)
(	85	,	1.271595	)
(	86	,	1.3039724	)
(	87	,	1.3445101	)
(	88	,	1.3690073	)
(	89	,	2.1153033	)
(	90	,	2.8706532	)
(	91	,	36.155743	)
	}; \label{plot_time_dec}
	\addlegendentry{Time (Dec)}		
	\end{axis}
	
	\begin{axis}[
		xmax=91,
		ymin=10,
		ymax=10e5,
		legend style={at={(0.95,0.05)},anchor=south east},
	        axis y line*=right,
	        ylabel near ticks, 
	        yticklabel pos=right,
	        axis x line=none,
		ymode=log,
		ylabel=Nodes]
	\addlegendimage{/pgfplots/refstyle=plot_time_gac}\addlegendentry{Time (GAC)}
	\addlegendimage{/pgfplots/refstyle=plot_time_dec}\addlegendentry{Time (Dec)}
	\addplot[color=red,dashed] coordinates {
(	1	,	833	)
(	2	,	833	)
(	3	,	834	)
(	4	,	834	)
(	5	,	834	)
(	6	,	834	)
(	7	,	842	)
(	8	,	864	)
(	9	,	864	)
(	10	,	864	)
(	11	,	864	)
(	12	,	888	)
(	13	,	888	)
(	14	,	888	)
(	15	,	1437	)
(	16	,	1437	)
(	17	,	1437	)
(	18	,	1437	)
(	19	,	1437	)
(	20	,	1437	)
(	21	,	1437	)
(	22	,	1437	)
(	23	,	1437	)
(	24	,	1437	)
(	25	,	1437	)
(	26	,	1437	)
(	27	,	1437	)
(	28	,	1437	)
(	29	,	1438	)
(	30	,	1445	)
(	31	,	1445	)
(	32	,	1445	)
(	33	,	1445	)
(	34	,	1445	)
(	35	,	1445	)
(	36	,	1445	)
(	37	,	1445	)
(	38	,	1445	)
(	39	,	1445	)
(	40	,	1445	)
(	41	,	1445	)
(	42	,	1637	)
(	43	,	1637	)
(	44	,	1637	)
(	45	,	1637	)
(	46	,	1637	)
(	47	,	1637	)
(	48	,	1637	)
(	49	,	1637	)
(	50	,	1637	)
(	51	,	1637	)
(	52	,	1637	)
(	53	,	1637	)
(	54	,	1637	)
(	55	,	1637	)
(	56	,	1637	)
(	57	,	1637	)
(	58	,	1642	)
(	59	,	1642	)
(	60	,	1682	)
(	61	,	1724	)
(	62	,	1724	)
(	63	,	2756	)
(	64	,	3431	)
(	65	,	3447	)
(	66	,	3456	)
(	67	,	3466	)
(	68	,	3469	)
(	69	,	4254	)
(	70	,	4254	)
(	71	,	4254	)
(	72	,	4254	)
(	73	,	4254	)
(	74	,	4254	)
(	75	,	4254	)
(	76	,	4254	)
(	77	,	5668	)
(	78	,	5866	)
(	79	,	5948	)
(	80	,	6135	)
(	81	,	6207	)
(	82	,	6211	)
(	83	,	6578	)
(	84	,	6748	)
(	85	,	6909	)
(	86	,	6948	)
(	87	,	7081	)
(	88	,	7210	)
(	89	,	7313	)
(	90	,	7313	)
(	91	,	12228	)
	}; \label{plot_node_gac}
	\addlegendentry{Nodes (GAC)}	
	\addplot[color=blue,dashed] coordinates {
(	1	,	833	)
(	2	,	833	)
(	3	,	834	)
(	4	,	834	)
(	5	,	834	)
(	6	,	834	)
(	7	,	842	)
(	8	,	864	)
(	9	,	864	)
(	10	,	864	)
(	11	,	864	)
(	12	,	888	)
(	13	,	888	)
(	14	,	888	)
(	15	,	1441	)
(	16	,	1441	)
(	17	,	1441	)
(	18	,	1444	)
(	19	,	1444	)
(	20	,	1444	)
(	21	,	1445	)
(	22	,	1445	)
(	23	,	1445	)
(	24	,	1445	)
(	25	,	1445	)
(	26	,	1445	)
(	27	,	1445	)
(	28	,	1445	)
(	29	,	1637	)
(	30	,	1637	)
(	31	,	1637	)
(	32	,	1637	)
(	33	,	1637	)
(	34	,	1637	)
(	35	,	1637	)
(	36	,	1637	)
(	37	,	1637	)
(	38	,	1637	)
(	39	,	1637	)
(	40	,	1637	)
(	41	,	1637	)
(	42	,	1637	)
(	43	,	1637	)
(	44	,	1637	)
(	45	,	1641	)
(	46	,	1642	)
(	47	,	1642	)
(	48	,	1724	)
(	49	,	1724	)
(	50	,	2360	)
(	51	,	2360	)
(	52	,	2480	)
(	53	,	2483	)
(	54	,	2483	)
(	55	,	2483	)
(	56	,	2715	)
(	57	,	2715	)
(	58	,	2715	)
(	59	,	2715	)
(	60	,	2715	)
(	61	,	2715	)
(	62	,	3189	)
(	63	,	7237	)
(	64	,	7237	)
(	65	,	7237	)
(	66	,	7237	)
(	67	,	7237	)
(	68	,	8569	)
(	69	,	8569	)
(	70	,	11014	)
(	71	,	11014	)
(	72	,	11014	)
(	73	,	11014	)
(	74	,	11014	)
(	75	,	11014	)
(	76	,	23444	)
(	77	,	23444	)
(	78	,	27584	)
(	79	,	27604	)
(	80	,	27604	)
(	81	,	27604	)
(	82	,	27604	)
(	83	,	27725	)
(	84	,	27741	)
(	85	,	27741	)
(	86	,	27741	)
(	87	,	27741	)
(	88	,	27741	)
(	89	,	27741	)
(	90	,	133662	)
(	91	,	626190	)
	}; \label{plot_node_dec}
	\addlegendentry{Nodes (Dec)}
	\end{axis}
\end{tikzpicture}
}
\caption{Solution times for the BNWP test bed}
\label{fig:bacp_time}
\end{figure}

\subsection{Determining confidence intervals for the multinomial distribution}

In the previous two sections the $\chi^2$ test statistic has been essentially employed as as a least square measure of discrepancy from a desired distributional form. In the context of these applications, $\alpha$ therefore represents a constraint ``softening'' coefficient, rather than a significance level. In this section, we introduce an application of a variant of the $\chi^2$ test statistical constraint in the context of the well-known problem of determining simultaneous confidence intervals for the multinomial distribution. In the context of this application, $\alpha$ retains its original nature of statistical significance level.

In statistics a multivariate generalisation of the $\chi^2$ test is the so-called ``score test,'' which can be used to carry out hypothesis testing on multivariate distributions (see \cite{citeulike:14214028}, chap. 5). In this section we shall concentrate on the multinomial distribution (see \cite{citeulike:14214023}, section 3). Consider a multinomial distribution with event probabilities $p_1,\ldots,p_k$, where $k$ is the number of categories, and $N$ trials. Let $x_1,\ldots,x_n$ be $n$ i.i.d. random variates and $c_1,\ldots,c_k$ be associated observed cell counts in a sample size of $N=\sum_{i=1}^k c_i$. The problem of determining simultaneous confidence intervals for $p_1,\ldots,p_k$ was developed in the Sixties by \cite{citeulike:14214031,10.2307/1266673,citeulike:14214026}.

The maximum likelihood estimators of $p_j$ are $\hat{p}_j=c_j/N$, $j=1,\ldots,k$. The random vector $\hat{p}\equiv(\hat{p}_1,\ldots,\hat{p}_k)$ is asymptotically distributed according to a multivariate normal distribution with mean vector $p\equiv(p_1,\ldots,p_k)$ and covariance matrix $\Sigma/N$ with elements $\sigma_{jj}=p_j(1-p_j)$ and $\sigma_{ij}=-p_ip_j$ for $i\neq j$. In what follows, we shall concentrate on the work of \cite{citeulike:14214026}, who discuss confidence intervals based on the quadratic form
\[N(\hat{p}-p)'\Sigma^{-1}(\hat{p}-p)\]
which is asymptotically distributed as a $\chi^2$ distribution with $k-1$ degrees of freedom. Let $1-\alpha$ be the desired confidence level, confidence intervals are obtained, for $j=1,\ldots,k$, as the two solutions of equation
\[N\frac{(\hat{p}_j-p_j)^2}{p_j(1-p_j)}=F^{-1}_{\chi^2_{k-1}}(1-\alpha)\]

The very same intervals can be easily computed via a simple variant of the model originally presented in Fig. \ref{model:chi_square}, in which Pearson's $\chi^2$ statistic is replaced by Quesenberry and Hurst's statistic. The revised model is shown in Fig. \ref{model:score_test}; constraint (3) ensures this model computes confidence interval lower bounds. 

\begin{figure}
\begin{center}
        \framebox{
        \begin{tabular}{ll}
        \mbox{\bf Constraints:}\\
        \multicolumn{2}{l}{
        ~~~~(1)~~$\textsc{bin\_counts}_{1,\ldots,k+1}(x_1,\ldots,x_n;c_1,\ldots,c_k)$}\\
        \multicolumn{2}{l}{
        ~~~~(2)~~$N\frac{(c_j/N-p_j)^2}{p_j(1-p_j)}= F^{-1}_{\chi_{k-1}^2}(1-\alpha)$\hspace{2em}$j=1,\ldots,k$}\\  
        \multicolumn{2}{l}{
        ~~~~(3)~~$p_j\leq c_j/N$\hspace{10em}$j=1,\ldots,k$}\\ 
        \\     
        \mbox{\bf Parameters:} \\
        ~~~~$F^{-1}_{\chi_{k-1}^2}$ 			&inverse $\chi^2$ distribution with\\ 
        					   				&$k-1$ degrees of freedom\\
        ~~~~$\alpha$ 						&target significance for the score test\\
        \\
        \mbox{\bf Decision variables:} \\
        ~~~~$x_1,\ldots,x_n$ 				&random variates\\
        ~~~~$c_1,\ldots,c_k$ 				&observed cell counts\\
        ~~~~$p_1,\ldots,p_k$ 				&lower bounds for multinomial event probabilities
        \end{tabular}
        }
    \caption{A score test statistical constraint decomposition to compute lower bounds of Quesenberry and Hurst's confidence intervals. To compute the respective confidence interval upper bounds constraint (3) should be replaced by $p_j\geq c_j/N$.}
    \label{model:score_test}
\end{center}    
\end{figure}

In \cite{10.2307/1266673} the authors discussed tighter version of Quesenberry and Hurst's intervals. This and other variants such as the one in \cite{citeulike:14214031} can be modelled by simply modifying the original statistic in constraint (2).

{\bf Example.} We consider $N=10$ i.i.d. observations drawn from a multinomial with event probability vector $p=\{0.3,0.3,0.4\}$. The ten observations are $x=\{1, 1, 2, 0, 1, 1, 1, 0, 2\}$, the associated cell counts are $c=\{3,5,2\}$. We set the target significance for the score test $\alpha=0.1$ (i.e. a confidence level $1-\alpha=0.9$); in Table \ref{tab:confidence_intervals} we compare confidence intervals obtained using Quesenberry and Hurst's closed form expressions and intervals obtained as solutions of our model (Fig. \ref{model:score_test}) based on the score test statistical constraint. 

In the example here presented the ten observations and the associated cell counts were scalar values. However, the constraint program in Fig. \ref{model:score_test} models random variates and bin counts as decision variables. This opens up opportunities for declarative modeling with applications in multiple domains \cite{rossi2014,citeulike:13771650,Pachet:2015:GNS:2832581.2832596,sony2016}.

\begin{table}[t]
\centering
\begin{tabular}{l|rr}
								&Quesenberry and Hurst's					&Score test decomposition\\
\hline							
$(p^{\text{lb}}_1,p^{\text{ub}}_1)$		&(0.0981, 0.6280)	&(0.0981,0.6280)\\
$(p^{\text{lb}}_2,p^{\text{ub}}_2)$		&(0.2192, 0.7808)	&(0.2192,0.7808)\\
$(p^{\text{lb}}_3,p^{\text{ub}}_3)$		&(0.0509, 0.5383)	&(0.0509,0.5383)
\end{tabular}
\caption{Confidence intervals for our numerical example}
\label{tab:confidence_intervals}
\end{table}

\section{Related works}\label{sec:related_works}

Constraint categories that are related to \textsc{bin\_counts} include counting constraints and value constraints. Given a set of decision variables, the \textsc{count} constraint can be exploited to constrain the number of variables which take a given scalar value; \textsc{among} \cite{citeulike:14183355} constrains the number of decision variables which take values contained within a given set of scalar values; essentially this constraint can be seen as a \textsc{bin\_counts} over a single bin. 
\textsc{interval\_and\_count} \cite{Cousin93} and \textsc{assign\_and\_counts} \cite{Beldiceanu:2007:GCC:1232658.1232664} deal with the allocation of tasks to bin. However, in both cases the semantics involve properties of the task --- such as being assigned a given colour --- and determines a single common bound on the number of item that are allocated to a bins, rather than counting, for each bin, the number of elements allocated to it. The 
\textsc{cumulative} constraint \cite{Aggoun:1993:ECO:2262436.2262990} enforces that at each point in time, the cumulated height of the set of tasks that overlap that point does not exceed a given limit; the \textsc{interval\_and\_sum} constraint, derived from the previous one, fixes the origins of a collection of tasks in such a way that, for all the tasks that are allocated to the same interval, the sum of the heights does not exceed a given capacity. In both these constraints all intervals have the same size and, once more, these constraints does not count, for each interval, the number of elements allocated to it, they enforce instead a common capacity limit.

The $\textsc{global\_cardinality}_{v_1,\ldots,v_K}$($x_1,\ldots,x_n;c_1,\ldots,c_m$) constraint \cite{charme89} requires that, for each $j=1,\ldots,m$, decision variable $c_j$ is equal to the number of variables $x_1,\ldots,x_n$ that are assigned scalar $v_j$. The \textsc{bin\_counts} constraint represents a generalisation of $\textsc{global\_cardinality}$ in which scalar values $v_1,\ldots,v_n$ are replaced by intervals representing bins. A GAC algorithm for the \textsc{global\_cardinality} constraint, which builds upon and generalises the results in \cite{REGIN:1994:FAC:2891730.2891786}, was discussed in \cite{Regin:1996:GAC:1892875.1892906}. Our GAC approach for \textsc{bin\_counts} generalises this latter discussion, since the approach in \cite{Regin:1996:GAC:1892875.1892906} does not reduce the domains of the count variables $c_j$.

Finally, \textsc{bin\_counts} can be used to express generalisations of constraints such as \textsc{spread} \cite{citeulike:13171963}
and \textsc{deviation} \cite{citeulike:14181517}. These generalisations can be used to model aspects of a distribution that go beyond moments such as mean and standard deviation. As we have shown with the $\chi^2$ test constraint, this paves the way to a range of applications in the context of statistical constraints.

\section{Conclusions}\label{sec:conclusions}

We discussed the \textsc{bin\_counts} constraint, which deals with the problem of counting the number of decision variables in a set which are assigned values that lie in given bins. We presented a decomposition and a GAC propagation strategy, as well as a decomposition for a new statistical constraint --- the $\chi^2$ test constraint --- based on \textsc{bin\_counts}. We discuss three applications of the $\chi^2$ test constraint: reformulations for the BACP and the BNWP, as well as an application in confidence interval analysis. In our computational study we illustrate the enhanced filtering achieved by our GAC propagation strategy over a constraint decomposition in the context of a set of randomly generated instances. This enhanced filtering led to order of magnitude improvements observed for search performance --- both in terms of computational time and number of nodes explored --- in the context of an existing CSPLib test bed for the BACP. A decomposition based on the \textsc{global\_cardinality} constraint led to superior performance in terms of computational time in the context of an existing CSPLib test bed for the BNWP; however, also in this case, a GAC propagation strategy led to stronger filtering. Finally, we presented an application of the $\chi^2$ test constraint in the context of a well-known problem from the confidence interval analysis literature: the problem of determining simultaneous confidence intervals for the multinomial distribution. Although this problem is well-known in the literature, to the best of our knowledge a declarative approach based on statistical constraints has never been presented before.

\bibliography{bin_counts}

\end{document}